\documentclass[letterpaper, 10 pt, conference]{ieeeconf}  %
\usepackage[T1]{fontenc}
\usepackage[latin9]{inputenc}
\usepackage{mathtools}

\usepackage{amsmath}
\usepackage{amsthm}
\usepackage{amssymb}
\usepackage{graphicx}
\usepackage{hyperref}
\usepackage{array}
\usepackage{balance}
\usepackage{marvosym}
\usepackage{comment}
\usepackage{algorithm}
\usepackage[noend]{algpseudocode}

\algrenewcommand\algorithmicindent{1.0em}

\makeatletter

\newtheorem{prop}{Proposition}

\newtheorem*{lem*}{Lemma}
\newtheorem{rem}{Remark}

\IEEEoverridecommandlockouts
\begin{document}
\title{Selection of Input Primitives for the Generalized Label Correcting Method}
\author{Brian Paden$^{\rm a}$ and Emilio Frazzoli.$^{\rm a}$\\
	\thanks{$\,^{\rm a}$The authors are with the Laboratory for Information and Decision Systems at MIT. Email: bapaden@mit.edu, frazzoli@mit.edu}}
\maketitle
\begin{abstract}
The generalized label correcting method is an efficient search-based approach to trajectory optimization.
It relies on a finite set of control primitives that are concatenated into candidate control signals.
This paper investigates the principled selection of this set of control primitives.
Emphasis is placed on a particularly challenging input space geometry, the $n$-dimensional sphere. 
We propose using controls which minimize a generalized energy function and discuss the optimization technique used to obtain these control primitives.
A numerical experiment is presented showing a factor of two improvement in running time  when using the optimized control primitives over a random sampling strategy.
\end{abstract}


\section{Introduction}

Kinodynamic motion planning and trajectory optimization problems consist of finding an open loop control signal from an infinite dimensional signal space which minimizes a cost functional. 
This challenging problem is approached by borrowing techniques from numerical analysis to approximate the input signal space by a subset over which an optimal solution can be efficiently computed.
A traditional approach is to approximate the signal space by a finite dimensional vector space. 
This enables the use of well developed finite-dimensional nonlinear optimization methods to compute locally optimal solutions to the approximated problem~\cite{betts1998survey}.
This is not always an acceptable strategy since many kinodynamic motion planning problems have unsatisfactory local minima.
The alternative is search-based methods which approximate the input signal space by strings of control signal primitives (or dynamically feasible trajectories) which can be efficiently searched using graph search techniques.
Some examples of search based methods include Lattice-planners~\cite{lindemann2003incremental}, the kinodynamic variant of the Rapidly Exploring Random Tree ($\rm RRT^*$)~\cite{karaman2010optimal}, the Stable Sparse RRT ($\rm SST$) method~\cite{li2016asymptotically}, and the generalized label correcting (GLC) method~\cite{paden2016generalized}.
There are additionally some hybrid approaches utilizing ideas from both finite-dimensional optimization and search-based methods~\cite{stoneman2014embedding,xie2015toward,choudhury2016regionally}.
The $\rm GLC$ method, which is the focus of this paper, is a general search-based algorithm for efficiently generating feasible trajectories solving a trajectory optimization or kinodynamic motion planning problem.
The advantage to this method is the ability to compute feasible trajectories and control signals whose cost approximates the globally optimal cost to the problem in finite time.
In contrast to related algorithms the $\rm GLC$ method relies on weaker technical assumptions and fewer subroutines such as a local point-to-point planning solution. 
%
%

In this paper we continue the development of the $\rm GLC$ method by investigating a selection strategy for the control primitives used by the algorithm.  
Without prior knowledge of the problem there is no reason to bias the selection of control primitives around any point in the input space.
This suggests evenly dispersing the control primitives on the set of allowable control inputs. 
One technique for obtaining approximately evenly dispersed control primitives is random sampling (cf. Figure \ref{fig:rando}).
This is essentially the approach taken with some randomized methods such as $\rm SST$.
In contrast, a more principled deterministic construction of control primitives can yield a more evenly distributed collection of points on the input space.

\begin{figure}
	\begin{centering}
		\includegraphics[width=0.5\columnwidth]{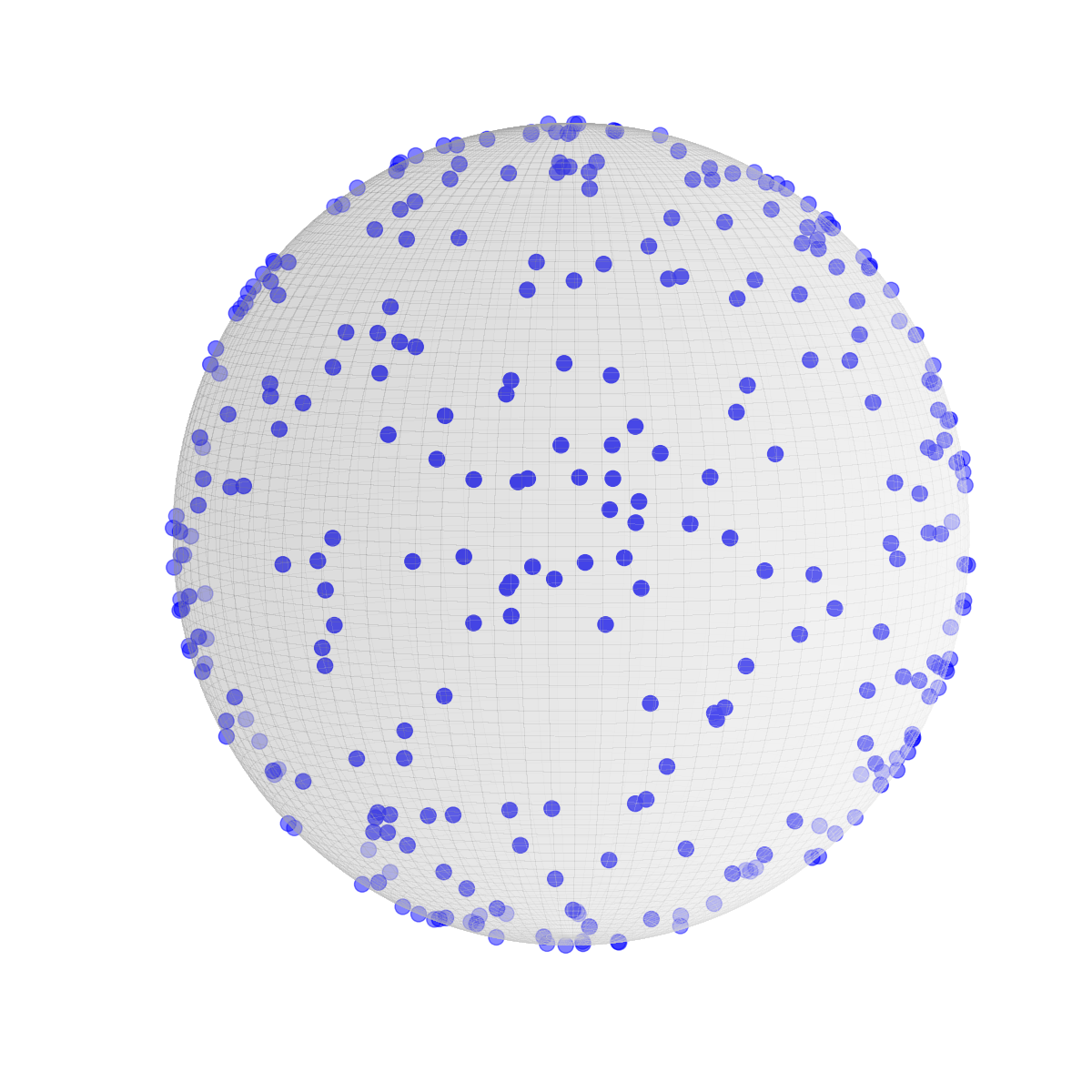}\includegraphics[width=0.5\columnwidth]{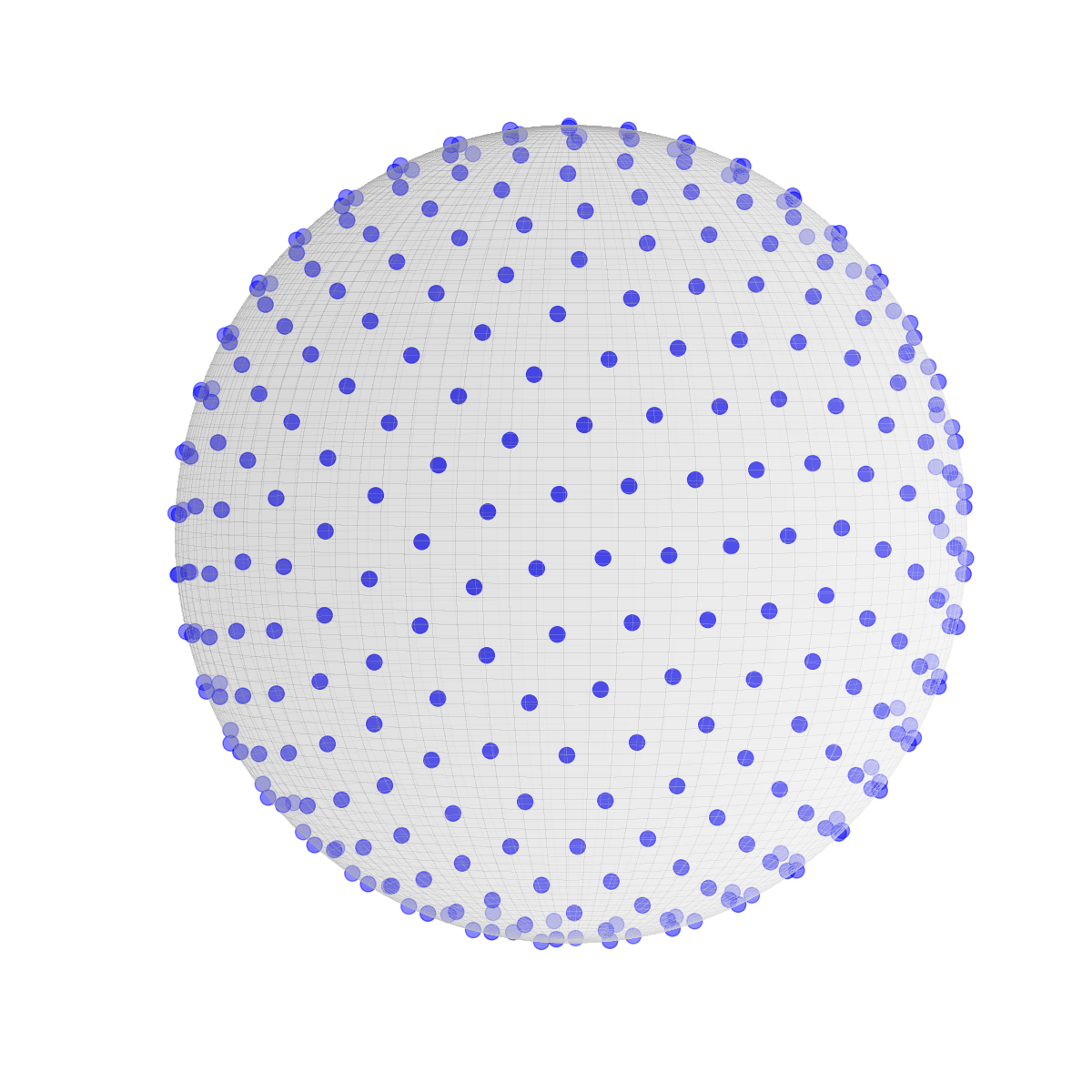}
		\par\end{centering}
	\caption{\label{fig:rando} Side-by-side comparison of the distribution of 500 points generated
		by randomly sampling from the uniform distribution (left) and by obtaining a local solution to the
		Thomson problem (right).}
	
\end{figure}
Sukharev grids~\cite{sukharev1971optimal} are an optimal (they minimize $L_{\infty}$ dispersion) arrangement of points on hypercubes and are easily computed.  
In contrast, control input spaces described by the $n$-dimensional sphere are frequently encountered in control input constrained systems and are a challenging geometry for distributing points evenly. 
The principal contribution of this paper is a procedure for computing optimal configurations of control primitives on the $n$-sphere along with an open source implementation \cite{t_solver}. 
This is accomplished by defining an interaction potential between the points representing control primitives and computing a local minima of the total energy.
This problem arises in many scientific fields and is classically known as Thomson's problem~\cite{saff1997distributing}.
Before describing the control primitive selection strategy in detail, a review of the kinodynamic motion planning problem is presented in Section \ref{sec:kino} together with a high level overview of the $\rm GLC$ method. 
This is followed by an empirical comparison of the proposed selection strategy with a randomized selection strategy in Section \ref{sec:demo}. 
The optimization, known as Thomson's problem, used to select control primitives is presented in Section \ref{sec:thomson_problem}. 
Lastly, Section \ref{sec:gp_method} discusses the nonlinear programming technique used to find locally optimal solutions to Thomson's problem.
\section{Kinodynamic Motion Planning\label{sec:kino}}

Kinodynamic motion planning is a form of open loop trajectory optimization with differential and point-wise constraints.
Differential constraints can be expressed by a classical nonlinear control system
\begin{equation}
\dot{x}(t)=f(x(t),u(t)),\label{eq:ode}
\end{equation}
with $x(t)\in \mathbb{R}^n$, $u(t)\in \Omega$, and $f:\mathbb{R}^n \times \Omega \rightarrow \mathbb{R}^n$.
The control input space $\Omega$ is a subset of $\mathbb{R}^m$.
A dynamically feasible trajectory over the time interval $[0,T]$ is a continuous time history of states $x:[0,T]\rightarrow \mathbb{R}^n$ for which there exists a measurable, essentially bounded control signal $u:[0,T]\rightarrow \Omega$ satisfying (\ref{eq:ode}) almost everywhere. 

The problem has three point-wise constraints.
The first is an initial state constraint, $x(0)=x_0$.
The second is a constraint enforced along the entire trajectory, $x(t)\in X_{free}$ for all $t\in[0,T]$.
Lastly, there is a terminal constraint, $x(T)\in X_{goal}$.
A trajectory that satisfies the differential and point-wise constraints is said to be a \textit{feasible} trajectory.

Next, a cost functional $J$ is used to measure the relative quality of trajectories, 
\begin{equation}
J(u)=\int_{[0,T]}g(x(t),u(t))\, d\mu(t).\label{eq:cost}
\end{equation}
In the above expression $x$ is the unique solution to (\ref{eq:ode}) with input $u$ and initial condition $x_0$. 
Since the minimum of this functional is not attained in general, the goal of computational methods for optimal kinodynamic motion planning is to return a sequence of feasible trajectories and control signals which converge to the optimal value of $J$ subject to the feasibility constraints.

\subsection{The Generalized Label Correcting Method}

The $\rm GLC$ method uses a single resolution parameter $R$ to balance an approximation of the control signal space and state space.
An approximate shortest path algorithm is applied to the approximation so that the output of the algorithm converges, in cost, to the optimal cost of the problem with increasing resolution. 
The control signal space is approximated by a tree of piecewise constant control signals taking values from a finite subset $\Omega_R$ of the input space $\Omega$.
The duration of constant control primitives is $1/R$ and $\Omega_R$ is assumed to converge to a dense subset of $\Omega$ as $R\rightarrow\infty$. 
The total duration of a control signal is limited to $h(R)/R$ where $h:\mathbb{N}\rightarrow \mathbb{R}$ defines a horizon limit that must grow unbounded as $R\rightarrow \infty$. 
This simple construction ensures an optimal control signal can be approximated arbitrarily well with sufficiently high resolution~\cite[Lemma 3]{paden2016generalized}.

The number of alternatives in 
the above construction grows exponentially with $h(R)$. 
Thus, the $\rm GLC$ method performs an approximate search of for the optimal control signal within this finite set. 
This is accomplished by defining a hyper-rectangular grid on the state space and considering controls producing trajectories terminating within the same hyper-rectangle as equivalent. 
This is denoted $u_1 \overset{R}{\sim} u_2  $.
The "label" for a hyper-rectangular region is the lowest cost signal producing a trajectory terminating in that region discovered by the search at any given iteration.
Like the approximation of the signal space, the grid is controlled by the resolution.
Each cell in the grid must be contained in a ball of radius $\eta(R)$ where $\eta(R)$ satisfies 
\begin{equation}
\lim_{R\rightarrow\infty}\frac{R}{L_f\eta(R)}\left(e^{\frac{L_{f}h(R)}{R}}-1\right)=0,\label{eq:partition_scaling}
\end{equation}
for a global Lipschitz constant $L_f$ on the system dynamics with respect to $x$.
Then if $u_1 \overset{R}{\sim} u_2$ and 
\begin{equation}\label{eq:prune}
J(u_{1})+\frac{\sqrt{n}}{\eta(R)}\frac{L_{g}}{L_{f}}\left(e^{\frac{L_{f}h(R)}{R}}-1\right) \leq J(u_{2}),
\end{equation}
the signal $u_2$ and subsequent concatenations of signals beginning with $u_2$ can be omitted or pruned from the search.
$L_g$ denotes a global Lipschitz constant for the running cost $g$ in (\ref{eq:cost}).
Among the remaining signals is a signal with approximately the optimal cost. 
With increasing resolution this signal converges to the optimal cost~\cite[Theorem 1]{paden2016generalized}.

As a simple illustration of the pruning operation consider a 2D single integrator with $\Omega=\{u\in \mathbb{R}^2:\,\Vert u \Vert_2 =1\}$ and a minimum time objective. Figure \ref{fig:demo} shows how the pruned subset explores the space effectively from the initial condition in the lower left corner. In comparison, an exhaustive search over all strings of control primitives evaluates many paths that zig-zag around the initial condition.  

\begin{figure}
	\vspace{0.075in}
	\begin{centering}
		\includegraphics[width=0.8\columnwidth]{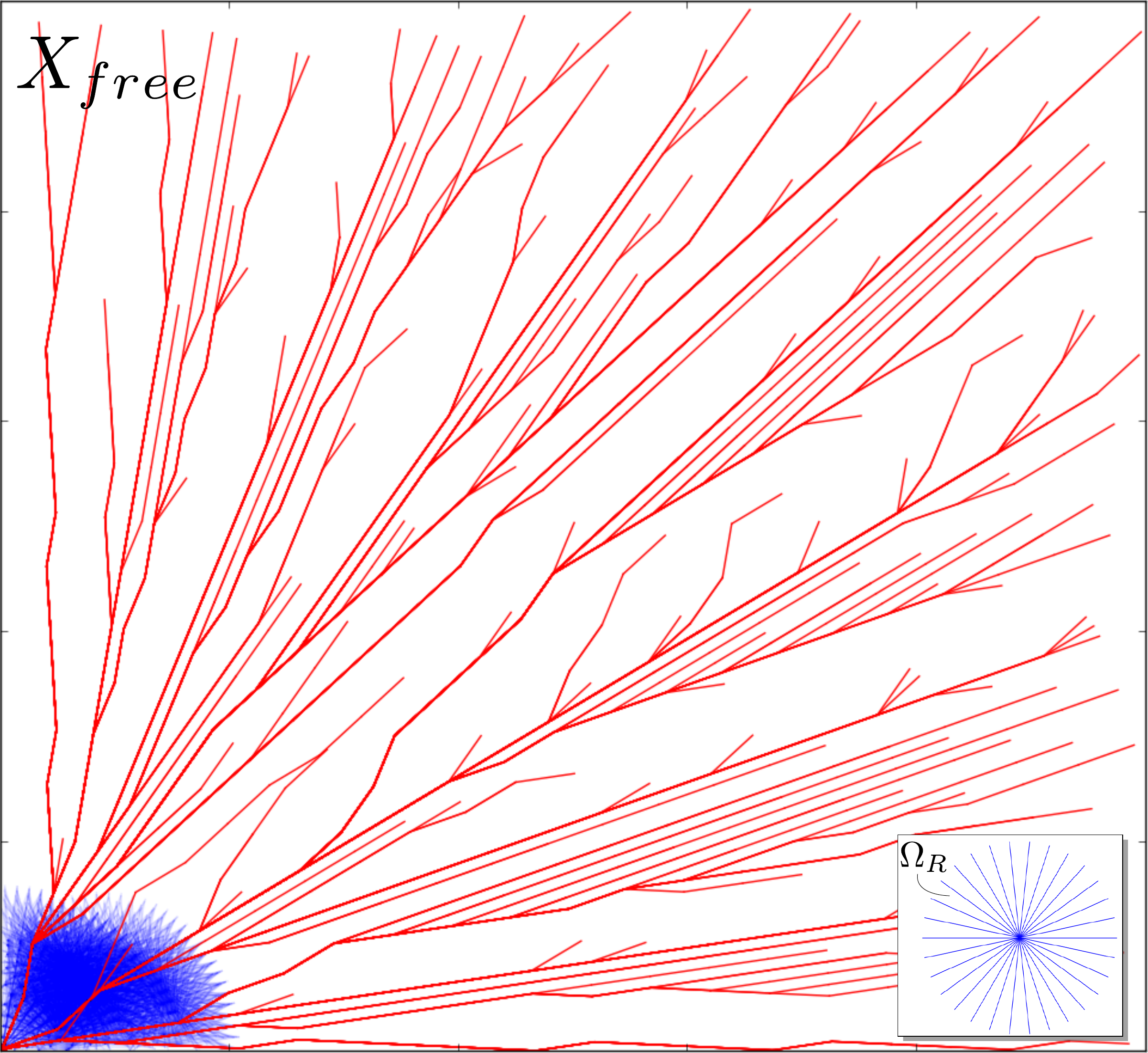}
		\par
		\caption{\label{fig:demo} 2D Single integrator example. The first 3000 iterations of a breadth first exhaustive search over all strings of control primitives is shown in blue. The subset of 416 strings satisfying the pruning condition and remaining in the illustrated rectangle are shown in red.}
	\end{centering}
\end{figure}

Algorithm \ref{ALG} describes the general search procedure which is a standard uniform cost search together with the pruning operation. 
A set $\mathcal{U}_{feas.}$ denotes signals producing trajectories remaining in $X_{free}$. Similarly, $\mathcal{U}_{goal}$ denotes signals producing feasible trajectories terminating in the goal.
The empty string, denoted $Id_{\mathcal{U}}$, has no cost and the $\rm NULL$ control has infinite cost.

The method $expand(u)$ returns the set of all signals consisting of $u$ concatenated with one more input primitive from $\Omega_R$.
A queue $Q$ contains candidate signals for future expansion.
The method $pop(Q)$ returns a signal $\hat{u}$ in $Q$ satisfying 
\begin{equation}
	\hat{u}\in \underset{w\in Q}{\rm argmin}\{J(w)\}.
\end{equation}
The method $find(w,\Sigma)$ returns a signal $z$ belonging to the same hyper-rectangle as $w$ from the set of labels $\Sigma$; if no such signal is present in $\Sigma$ the method $find(w,\Sigma)$ returns $\rm NULL$. The method $depth(w)$ returns the number of piecewise constant segments making up $w$. If $z$ prunes $w$ in the sense of equation (\ref{eq:prune}) we write $z\prec_R w$. 

\begin{algorithm} 
	\begin{algorithmic}[1]
		\State $Q\leftarrow \{Id_\mathcal{U}\},\,\Sigma \gets \emptyset,\,S \gets \emptyset$  
		\While {$Q\neq \emptyset$}        
		\State $u \gets pop(Q)$   
		\State $S \gets expand(u)$   
		\For{$w \in S$} 	
		\If{$w \in \mathcal{U}_\mathrm{goal} $}         	  
		\State \Return $(J(w),w)$ 	
		\EndIf       
		\State $z = find(w,\Sigma)$ 	  
		\If{$(w \notin \mathcal{U}_{feas.} \vee (z \prec_R w) \vee depth(w) \geq h(R))$}  	    
		\State $S \gets S\setminus \{w\}$  	  
		\ElsIf{$J(w)<J(z)$} 	  
		\State $\Sigma \gets (\Sigma \setminus \{z\}) \cup \{w\}$ 	  
		\EndIf 
		\EndFor       
		\State $Q \gets Q \cup S$ 
		\EndWhile  
		\State 
		\Return $(\infty,\ensuremath{\mathrm{NULL}})$ 
	\end{algorithmic} \caption{\label{ALG} Generalized Label Correcting ({\rm GLC}) Method} 
\end{algorithm}
%

%
%
 
\section{Motivating Example}\label{sec:demo}

\begin{figure*}
	\vspace{0.075in}
	\begin{centering}
		\includegraphics[width=1.0\textwidth]{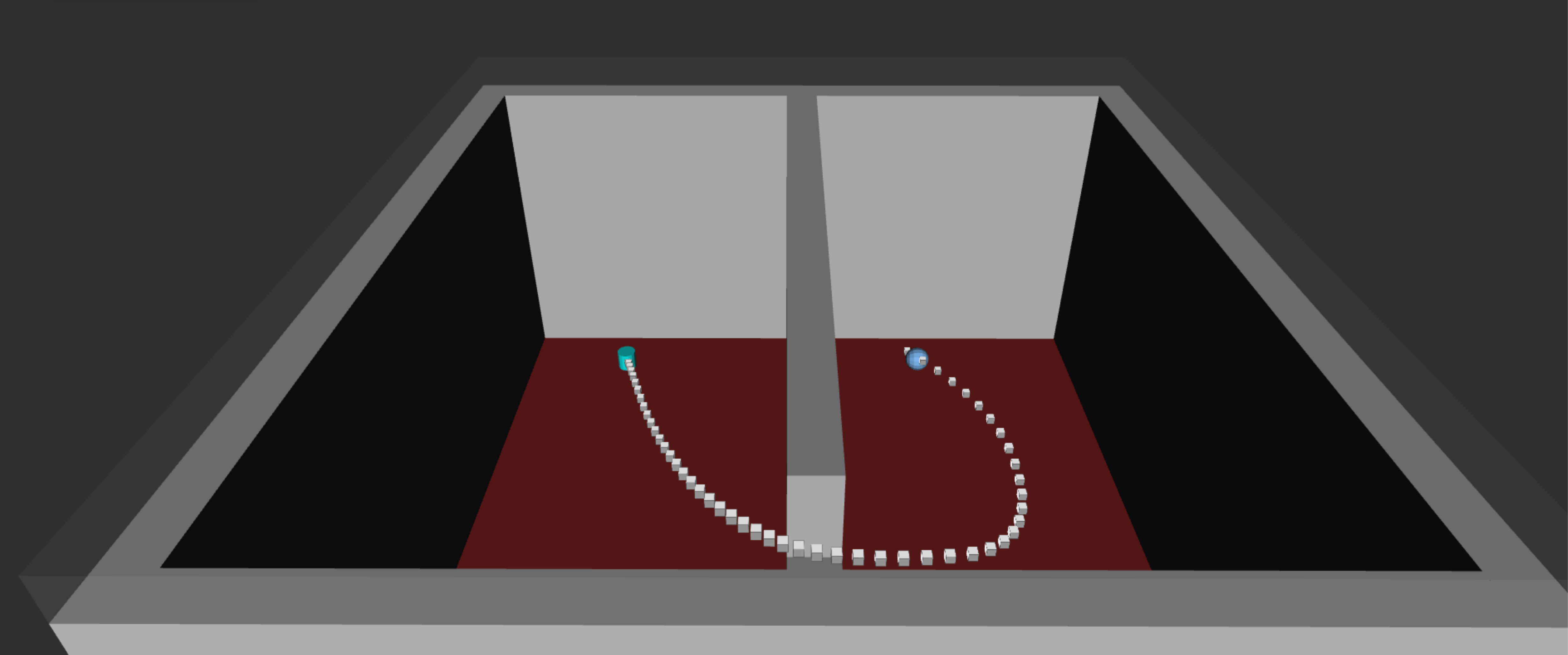}
		\par\end{centering}
	
	\caption{Visualization of the trajectory generated by the GLC method using the optimized input space approximation. The point robot accelerates from the initial stationary position (left) directly to the window followed by a wide cornering maneuver and continued acceleration towards the goal (right). Note that the goal set leaves the terminal velocity unspecified. Thus, the optimal trajectory intercepts the goal at high speed which causes the asymmetry in the solution depicted. \label{fig:quad}}
	
\end{figure*}

Consider an agile aerial robot navigating an indoor environment. 
The environment, depicted in Figure \ref{fig:quad}, consists of two $5m\times 5m\times 10m$ rooms connected by a $1m\times 1m$ window in an upper corner of each room. 
The task is to plan a dynamically feasible and collision free trajectory between a starting state and a goal set in minimum time.
The robot is modeled with six states; three each for position and velocity. 
The mobility of the robot is described by the following equations
\begin{equation}\arraycolsep=1.4pt\def\arraystretch{1.5}
\begin{array}{rcl}
\dot{x} & = & v,\\
\dot{v} & = & -0.1v\Vert v\Vert_{2}+5u.
\end{array}\label{eq:dynamics}
\end{equation}
The states are $x,v\in\mathbb{R}^{3}$ and the control is $u\in\mathbb{R}^{3}$.
In this representation the zero control is defined about a hover state negating the effect of gravity.
The term $-0.1v\Vert v\Vert_{2}$ reflects a quadratic aerodynamic drag, and the control $u$ is a thrust vector
which can be directed in any direction. 
The control is limited to a maximum thrust which is modeled by the constraint $\Vert u\Vert_{2}\leq1$ so that the robot's acceleration is limited to $5\,m/s^2$ and speed is limited to $\sqrt{50}\,m/s$.
It follows from Pontryagin's minimum principle \cite{athans2013optimal}
that the minimum time objective will yield saturated control inputs at all times,  $\Vert u(t)\Vert_{2}=1$ (i.e. the control is restricted to the sphere).
An implementation of the $ \rm GLC$ method in C++ was run on an Intel i7 processor at 2.6GHz. 
Parameters of the algorithm are provided in Table \ref{tab:params}.
In the first set of trials randomly generated control primitives on the $2$-sphere are obtained by sampling from the uniform distribution. 
In the second set of trials minimum energy arrangements of points were used.
Figure \ref{fig:rando} illustrates configurations of 500 points generated by the two strategies.
The average running time and solution cost of trajectories returned by the $\rm GLC$ method are summarized in Figure \ref{fig:run_time}.
We observe that the optimized input approximation strategy improves the running time required to obtain a trajectory of a given cost by roughly a factor of two for this problem.
\begin{table}
	\vspace{0.075in}
	\begin{center}
		\begin{tabular}{ | m{5.0cm} | m{2cm}|} 
			\hline
			Resolution range ($R$) & $R=8,9,...,13$ \\
			\hline
			Horizon limit ($h(R)$) & $10R\log(R)$ \\
			\hline
			State space partition scaling ($\eta(R)$)& $65R^{3/2}$ \\ 
			\hline
			Control primitive duration ($c/R$) & $10/R$\\ 
			\hline
			Number of controls from $n$-sphere ($\Omega_R$) & $\left\lfloor 3R^{3/2} \right\rfloor$ \\
			\hline 
		\end{tabular}
		\vspace{0em}
		\caption{Tuning parameter selection for the $\rm GLC $ method. \label{tab:params}}
	\end{center}
\end{table}
\begin{figure}
	\begin{centering}
		\includegraphics[width=1.05\columnwidth,trim={0.70cm 0.0cm 0.7cm 0.7cm},clip]{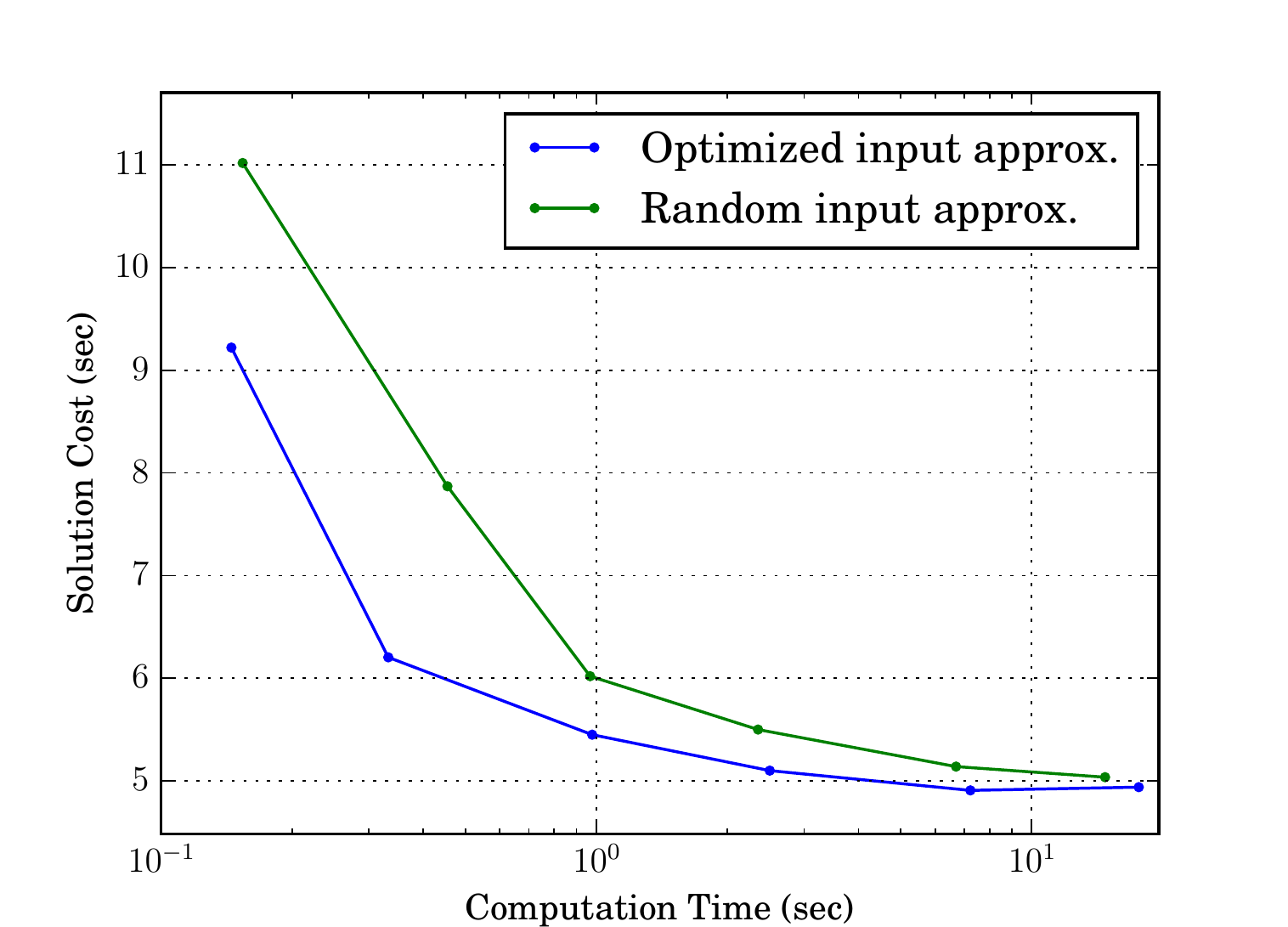}
		\par
		\caption{Average cost and running times (10 trials each) of solutions returned by the $\rm GLC$ method for the two input approximation schemes. The optimized selection (blue) improves the running time required to obtain a solution of a given cost over the random selection (green) by roughly a factor of two. \label{fig:run_time}}
	\end{centering}
\end{figure}

\section{Thomson's Problem\label{sec:thomson_problem}}

We have proposed using minimum energy configurations of points on the $n$-sphere as the control primitives for the $\rm GLC$ method. This is classically referred to as Thomson's problem.
This section provides a more detailed description of the problem and existence of optimal solutions.
Thomson's problem is an optimization problem seeking to find a minimum energy configuration
of $N$ points in $\mathbb{R}^{n}$ whose Euclidean norm
is $1$. 
The generalized energy is given by a superposition of pair-wise interactions parameterized by scalar $\alpha$.
The energy between two points $p_{1},p_{2}\in\mathbb{R}^{n}$ is given by $\left\Vert p_{1}-p_{2}\right\Vert _{2}^{\alpha}$
if $\alpha\neq0$, and $\log\left(\left\Vert p_{1}-p_{2}\right\Vert _{2}^{-1}\right)$
if $\alpha=0$. 

The configuration of the $N$ points is represented by a matrix $X\in\mathbb{R}^{n\times N}$ with each column representing coordinates of a point in $\mathbb{R}^{n}$.
The notation $X_{:,j}$ will be used to identify the coordinates of
the $j^{th}$ point in $\mathbb{R}^{n}$. 
In our analysis we need to distinguish between the Frobenius and Euclidean norms and inner products,
\begin{equation}\arraycolsep=1.4pt\def\arraystretch{1.5}
\begin{array}{rcl}
\left\langle X,Y\right\rangle_F & \coloneqq & {\rm Tr}(XY^{T}),\\
\left\Vert X \right\Vert_F  & \coloneqq &  \sqrt{Tr(XX^T)},\\
\left\langle x,y \right\rangle_2  & \coloneqq &  x^Ty,\\
\Vert x \Vert_2  & \coloneqq &  \sqrt{x^Tx}.
\end{array}\label{eq:norm}
\end{equation}
The set of feasible configurations $S$ is defined 
\begin{equation}
S\coloneqq\left\{ X\in\mathbb{R}^{n\times N}:\:\left\Vert X_{:,j}\right\Vert _{2}=1,\qquad j=1,...,N\right\} .
\end{equation}
It follows from this definition that
\begin{equation}
\left\Vert X\right\Vert _{F}=\sqrt{N}\qquad \forall X\in S.
\end{equation} 
Using the above notation, the total generalized energy is given by
\begin{equation}\arraycolsep=1.4pt\def\arraystretch{1.5} \label{eq:totalDist}
\begin{array}{ll}
E_{\alpha}(X) = \overset{N}{\underset{i=1} {\sum}} \underset{j<i}{\sum}\left(\left\Vert X_{:,i}-X_{:,j}\right\Vert _{2}^{\alpha}\right) & {\rm if}\, \alpha \neq 0, \\
E_{\alpha}(X) = \overset{N}{\underset{i=1} {\sum}} \underset{j<i}{\sum} \log \left(\left\Vert X_{:,i}-X_{:,j}\right\Vert _{2}^{-1}\right)  & {\rm if}\, \alpha = 0,

\end{array}
\end{equation}
and the optimization objective is
\begin{equation}\arraycolsep=1.4pt\def\arraystretch{1.5}
\begin{array}{l}
\underset{X\in S}{\min}E_{\alpha}(X)\qquad{\rm if}\:\alpha\leq0,\\
\underset{X\in S}{\max}E_{\alpha}(X)\qquad{\rm if}\:\alpha>0.
\end{array}\label{eq:prblm}
\end{equation}

\subsection{Existence of Optimal Configurations}
For $\alpha >0$, the energy is a sum of differentiable functions. Thus, the energy is also differentiable and continuous. 
The set of configurations $S$ is compact in $(\mathbb{R}^{n\times N}, \Vert \cdot \Vert_F$).
Then by Weierstrass' theorem the maximum value is attained on $S$. 

For $\alpha \leq 0$ the continuity of $E_{\alpha}$ is broken as a result of the negative exponent and becomes unbounded from above;
configurations with overlapping points have infinite energy. 
However, $E_{\alpha}$ remains continuous on any subset in which it is bounded.
Take $X_{0}\in S$ such that $E_\alpha (X_0)<\infty$, the set $\left\{ X\in \mathbb{R}^{n\times N}:\,E_{\alpha}(X)\leq E_{\alpha}(X_{0})\right\} $ is closed and by construction contains a minimizer over $S$ if one exists. 
Since $S$ is compact, $S\cap \{X \in \mathbb{R}^{n \times N} :\:f(X)\leq f(X_{0})\} $
is also a compact set on which $E_{\alpha}$ is continuous. 
Thus, the minimum over this subset is attained and is the minimum over $S$.

\section{The Gradient Projection Method}\label{sec:gp_method}

This section reviews the gradient projection method with the adaptation of the Armijo rule proposed by Bertsekas in \cite{gpConvergence} for closed convex sets. 
We then address applying this method to the Thomson problem.

Consider a general minimization problem where the feasible set $C$ is a nonempty \textit{closed convex} subset of $\mathbb{R}^{m}$.
%
The gradient of $f$ at $x$ is denoted $\nabla f(x)$, and the projection of $x\in\mathbb{R}^{m}$ into $C$ is denoted $[x]^{+}$ and satisfies 
\begin{equation}
\left\{ [x]^{+}\right\} =\underset{y\in C}{{\rm argmin}}\left\{ \Vert y-x\Vert_{2}\right\} .\label{eq:proj}
\end{equation}
For a closed and convex $C$, there is a unique $y\in C$ minimizing
$\Vert y-x\Vert_{2}$.
The convergence of the method relies on a non-expansiveness property, 
\begin{equation}
	\left\Vert [x]^+-[y]^+ \right\Vert_2 \leq \left\Vert x^+-y^+ \right\Vert_2, 
\end{equation}
which requires that $C$ be closed and convex.
Using (\ref{eq:proj}), the recursion of the gradient projection method
is of the form
\begin{equation}
x_{k+1}=[x_{k}-\gamma_{k}\nabla f(x)]^{+}.\label{eq:gd}
\end{equation}
The subsequent point $x_{k+1}$ is obtained by moving along the direction of steepest descent scaled by a step-size $\gamma_{k}$, and then projecting the result into $C$. 
In contrast to the gradient descent step for the classical Armijo rule, which is taken along the ray through $x_{k}$ in the direction of steepest descent, the gradient projection step is taken along the projection of that ray onto $C$. 
Three tuning parameters define the step size selection; $s>0$, $\sigma>0$, and $\beta\in(0,1)$. 
The step size is given by $\alpha_{k}=s\beta^{m}$ where $m$ is the smallest natural number
such that 
\begin{equation}\arraycolsep=1.4pt\def\arraystretch{1.5}
\begin{array}{l}
f\left(x_{k}\right)-f\left([x_{k}-s\beta^{m}\nabla f(x_{k})]^{+}\right)\geq\\
\left\langle \sigma\left(\nabla f(x_{k})\right) , \left(x_{k}-[x_{k}-s\beta^{m}\nabla f(x_{k})]^{+}\right) \right\rangle.
\end{array}\label{eq:armijo}
\end{equation}
In contrast to carrying out an exact minimization along the steepest descent direction, the Armijo-rule has much less computational overhead and simply finds a step size with a sufficient decrease in $f$.
Intuitively, $s$ is the initial large step size which is rapidly reduced as $m$ is increased from $0$.

It was shown in \cite{gpConvergence} that limit points $x^*$ of the sequence $\{x_{k}\}$ produced by (\ref{eq:gd}) satisfy the necessary (but not sufficient) condition for local optimality:
\begin{equation}\label{eq:opt_cond}
	\left\langle \nabla f(x^*),(x-x^*) \right\rangle_2\geq 0,\qquad \forall x\in C.
\end{equation} 
It follows from (\ref{eq:armijo}) that $f(x^{k+1})<f(x^k)$ so these limit points are generally local minima. 

\subsection{Application to the Thomson Problem}

To adapt the standard theory, we have to address the issue that the feasible set $S$ for the Thomson problem is not convex. 
We replace feasible set $S$ by its convex hull to ensure that the gradient projection iteration (\ref{eq:gp_step}) converges to a local solution on the convex hull.
We then prove that $S$ is invariant under the gradient projection iteration so an initial configuration of points on $S$ with bounded energy will converge to a local solution on $S$.

Let $conv(S)$ denote the convex hull of $S$. 
While $\left\Vert X\right\Vert _{F}=\sqrt{N}$ for every $X\in S$, we now have $\left\Vert X\right\Vert _{F}\leq\sqrt{N}$ for every $X\in conv(S)$. 

The relaxed problem is 
\begin{equation}\arraycolsep=1.4pt\def\arraystretch{1.5}
\begin{array}{l}
\underset{X\in conv(S)}{\min}E_{\alpha}(X)\qquad{\rm if}\:\alpha\leq0,\\
\underset{X\in conv(S)}{\max}E_{\alpha}(X)\qquad{\rm if}\:\alpha>0.
\end{array}\label{eq:cvx_relaxation}
\end{equation}
This problem admits optimal values for $X$ by the same argument as the original problem. 
The projection onto $B$ is given by 
\begin{equation}\arraycolsep=1.4pt\def\arraystretch{1.5}
\left([X]^{+}\right)_{:,j}\coloneqq\left\{ \begin{array}{cc}
\frac{X_{:,j}}{\left\Vert X_{:,j}\right\Vert _{2}}, & {\rm if}\;\left\Vert X_{:,j}\right\Vert _{2}>1\\
X_{:,j} & {\rm otherwise}
\end{array}\right..
\end{equation}
Note that if $X\in conv(S)$, the projection is the identity map.
However, if $X\notin conv(S)$, the projection takes $X$ into $S$ %
That is 
\begin{equation}\arraycolsep=1.4pt\def\arraystretch{1.5}
\begin{array}{ll}
\left[ X \right]^{+}=X & \quad \forall X \in conv(S),\\
\left[ X \right]^{+} \in S & \quad \forall X \notin conv(S).
\end{array}\label{eq:projection_props}
\end{equation}

The gradient projection iteration is then 
\begin{equation}\arraycolsep=1.4pt\def\arraystretch{1.5}
\begin{array}{c}
X^{k+1}=\left[X^{k}-\gamma_{k}\nabla E(X^{k})\right]^{+}\qquad{\rm if}\:\alpha\leq 0,\\
X^{k+1}=\left[X^{k}+\gamma_{k}\nabla E(X^{k})\right]^{+}\qquad{\rm if}\:\alpha>0.
\end{array}\label{eq:gp_step}
\end{equation}
The existence of optimal solutions to the objective over $conv(S)$
together with the standard theory ensures that limit points of the iteration will satisfy the optimality condition (\ref{eq:opt_cond}).
However, we are not interested in solutions on $conv(S)\setminus S$.
To ensure that the recursion converges to a stationary point on $S$, we make sure the initial configuration is in $S$. 
The justification for this is provided below. 
\begin{prop}
$S$ is an invariant set under the gradient projection iteration.\end{prop}
\begin{proof}
Suppose $\alpha<0$ and $X\in S$ (The essentially identical derivations for $\alpha>0$ and $\alpha=0$
are omitted for brevity). 
We will first show that that $X-\gamma\nabla E_{\alpha}(X)\notin conv(S)$ for any $\gamma >0$. 
Then by (\ref{eq:projection_props}) we will have $\left[X-\gamma\nabla E_{\alpha}(X)\right]^{+}\in S$.
Select an index $k$ and consider the motion of the coordinates $X_{:,k}$
in a step of the gradient projection iteration. The partial derivative
with respect to $X_{:,k}$ is 
\begin{equation}
\frac{\partial E(X)}{\partial X_{:,k}}=\sum_{j\neq k}\alpha\left(X_{:,k}-X_{:,j}\right)\left\Vert X_{:,k}-X_{:,j}\right\Vert _{2}^{\alpha-1}.
\end{equation}
Since $X_{:,j}\neq X_{:,k}$ and $\left\Vert X_{:,k}\right\Vert_2 =\left\Vert X_{:,k}\right\Vert_2 $,
we have the inequality 
\begin{equation}
\left\langle X_{:,k},\left(X_{:,k}-X_{:,j}\right)\right\rangle_2 >0,\label{eq:points_out}
\end{equation}
which is derived in the appendix. 
Since $\alpha<0$ and equation (\ref{eq:points_out}) is true for all $j\neq k$ we obtain 
\begin{equation}
\left\langle X,-\frac{\partial E(X)}{\partial X}\right\rangle_F >0.\label{eq:outward_descent}
\end{equation}
The interpretation (\ref{eq:outward_descent}) is that the steepest descent direction is directed out of $conv(S)$. 
For any step size $\gamma > 0$ we have
\begin{equation}\arraycolsep=1.4pt\def\arraystretch{1.5}
\begin{array}{l}
\left\Vert X-\gamma\nabla E_{\alpha}(X)\right\Vert _{F}\\
=\sqrt{\left\Vert X\right\Vert _{F}^{2}+2\left\langle X,-\gamma\nabla E_{\alpha}(X)\right\rangle_F +\left\Vert -\gamma\nabla E_{\alpha}(X)\right\Vert _{F}^{2}}\\
>\sqrt{\left\Vert X\right\Vert _{F}^{2}+2\left\langle X,-\gamma\nabla E_{\alpha}(X)\right\rangle_F }\\
>\left\Vert X\right\Vert _{F}
\end{array}\label{eq:norm_increases}
\end{equation}
By assumption $X\in S$ so $\left\Vert X\right\Vert _{F}=\sqrt{N}$
and $\left\Vert X-\gamma\nabla E_{\alpha}(X)\right\Vert _{F}>\sqrt{N}$
by equation (\ref{eq:norm_increases}). 
Thus, $\left\Vert X-\gamma\nabla E_{\alpha}(X)\right\Vert _{F}\notin conv(S)$ so in reference to (\ref{eq:projection_props}) the projection will take $ X-\gamma\nabla E_{\alpha}(X)$ into $S$ which is the stated result.
\end{proof}
In contrast, if $X\notin S$ it is not necessarily true that successive iterations of the gradient projection map will converge to $S$. 
It is not difficult to construct fixed points of the map on $conv(S)\setminus S$.
\subsection{Open Source Implementation}
A lightweight open source implementation in C++ has been made available \cite{t_solver}. 
The code has no external dependencies so that it can be put into use quickly and is easily integrated into larger projects. 

The initial configuration is sampled randomly from the uniform distribution on $S$. 
The method terminates at iteration $k$ if $\left|E(X^{k+1})-E(X^{k})\right|<\varepsilon_{tol}$ or $k=k_{max}$. 
A configuration file allows the user to specify $\varepsilon_{k}$ and $k_{max}$ as well as the number of points $N$, the dimension of the space $n$ and the power law in the generalized energy $\alpha$.
Additionally, the user can specify the Armijo step parameters $\sigma$, $\beta$, and $s$.

Figure \ref{fig:Visualizations} illustrates several point configurations generated by the released code in $\mathbb{R}^{3}$ for $\alpha=-1$ and various $N$.
%
\begin{figure}[htb]
	\begin{centering}
		\includegraphics[width=.8\columnwidth]{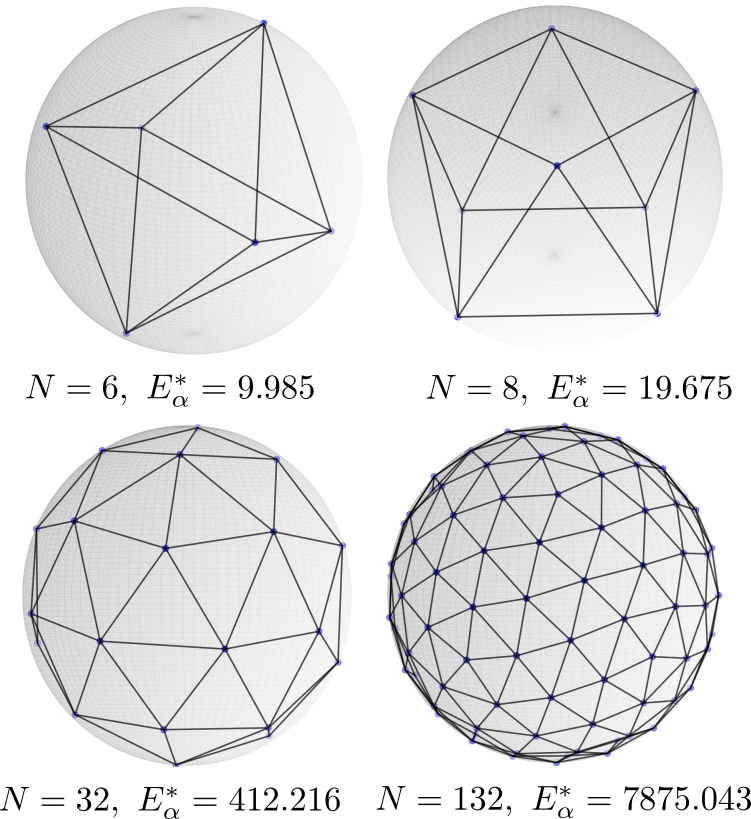}
		\par\end{centering}
	
	\centering{}\caption{\label{fig:Visualizations}Visualizations of various locally optimal configurations for the $2$-sphere with $\alpha=-1$. The objective values $E_{\alpha}^{*}$
		above coincide with the best known values found in \cite{erber1991equilibrium} and \cite{altschuler1997possible}.}
\end{figure}
%

%
\section{Conclusion}\label{sec:conclusion}
This paper demonstrated that an optimized selection of control primitives with respect to a general energy function improved performance of the $\rm GLC$ method by a factor of two in comparison to randomly sampled controls. 
Optimization of the control primitives was addressed with the gradient projection method.
While the resulting optimization does not meet the standard assumptions of the gradient projection method, a rigorous analysis showed that it remains applicable to this problem with an appropriately selected initial configuration of control inputs. 
An open source implementation of the gradient projection method applied to Thomson's problem has been made available to generate control primitives on the $n$-sphere.
\bibliographystyle{ieeetr}
\bibliography{IEEEabrv,references}

\appendix{}

The strict positivity of $\left\langle X_{:,j^{*}},\left(X_{:,j^{*}}-X_{:,k}\right)\right\rangle_2 $ in (\ref{eq:points_out}) is a consequence of the following Lemma which is true for any inner product space. 

Recall that $X_{:,j^{*}}\neq X_{:,k}$ and $\left\Vert X_{:,k}\right\Vert _{2} = \left\Vert X_{:,j^{*}}\right\Vert _{2}$.
\begin{lem*}
If $y\neq x$ and $\Vert y\Vert \leq \Vert x \Vert$, then $\left\langle x,x-y\right\rangle >0$. 
\end{lem*}
\begin{proof}
We have the strict inequality $0<\left\Vert x-y\right\Vert^{2}$
since $y\neq x$. 
Then 
\[
\begin{array}{rcl}
0 & < & \left\Vert x-y\right\Vert^{2}\\
 & = & \left\langle x,x\right\rangle -2\left\langle x,y\right\rangle +\left\langle y,y\right\rangle \\
 & \leq & 2\left\langle x,x\right\rangle -2\left\langle x,y\right\rangle ,
\end{array}
\]
where we used $\Vert y\Vert \leq \Vert x\Vert$ in the last
step. 
Rearranging the expression yields 
\[
\begin{array}{rcl}
0 & < & \left\langle x,x\right\rangle -\left\langle x,y\right\rangle \\
 & = & \left\langle x,x-y\right\rangle ,
\end{array}
\]
which is the desired inequality.\end{proof}

\end{document}